\documentclass[conference]{IEEEtran}
\usepackage{amsmath,amsfonts,amssymb,latexsym,amsthm,color} 
\usepackage{comment}

\usepackage[utf8]{inputenc} 
\usepackage[T1]{fontenc}    
\usepackage{hyperref}       
\usepackage{url}            
\usepackage{booktabs}       
\usepackage{amsfonts}       
\usepackage{nicefrac}       
\usepackage{microtype}      
\usepackage{graphicx}
\usepackage{algorithm}
\usepackage{algpseudocode}
\usepackage{cite}
\usepackage{subcaption}

\usepackage{graphicx} 

\usepackage{amsmath,amssymb,latexsym,amsthm,color}
\usepackage{bbm}
\usepackage{dsfont}
\newtheorem{thm}{Theorem}%
\newtheorem{defi}{Definition}%
\newtheorem{prop}{Proposition}
\newtheorem{lemma}{Lemma}

\def \1{{\mathds{1}}}

\def \bx{\boldsymbol{x}}

\def \bp{\boldsymbol{p}}

\def \phat{\widehat{\boldsymbol{p}}}
\def \qhat{\widehat{\boldsymbol{q}}}

\def\R{{\mathbb R}}

\usepackage{amsmath}

\newcommand{\cC}{\mathcal{C}}
\newcommand{\cW}{\mathcal{W}}
\newcommand{\cT}{\mathcal{T}}
\newcommand{\cQ}{\mathcal{Q}}
\newcommand{\cI}{\mathcal{I}}
\newcommand{\cO}{\mathcal{O}}

\newcommand{\PP}{\mathbb{P}}

\newcommand{\N}{\mathbb{N}}

\DeclareMathOperator{\diam}{diam}

\title{Exact Minimum-Volume Confidence Set Intersection for Multinomial Outcomes}

\author{
\IEEEauthorblockN{Heguang Lin$^{1}$}
\IEEEauthorblockA{heguangl@scripps.edu}
\and
\IEEEauthorblockN{Binhao Chen$^{2}$}
\IEEEauthorblockA{binhao\_chen@brown.edu}
\and
\IEEEauthorblockN{Mengze Li$^{3}$}
\IEEEauthorblockA{mli562@wisc.edu}
\and
\IEEEauthorblockN{Daniel Pimentel-Alarc\'on$^{3}$}
\IEEEauthorblockA{pimentelalar@wisc.edu}
\and
\IEEEauthorblockN{Matthew L. Malloy$^{3}$}
\IEEEauthorblockA{matthew.malloy@wisc.edu}
\thanks{
$^{1}$The Scripps Research Institute, La Jolla, CA, USA. 
$^{2}$Brown University, Providence, RI, USA.
$^{3}$University of Wisconsin–Madison, Madison, WI, USA. 
}
}

\IEEEoverridecommandlockouts
\onecolumn

\begin{document}

\maketitle

\begin{abstract} 

Computation of confidence sets is central to data science and machine learning, serving as the workhorse of A/B testing and underpinning the operation and analysis of reinforcement learning algorithms \cite{jamieson2014lil}. Among all valid confidence sets for the multinomial parameter, \emph{minimum-volume confidence sets} (MVCs) are optimal in that they minimize average volume \cite{malloy2021ISIT}, but they are defined as level sets of an exact $p$-value that is discontinuous and difficult to compute. Rather than attempting to characterize the geometry of MVCs directly, this paper studies a fundamental and practically motivated decision problem: given two observed multinomial outcomes, can one certify whether their MVCs intersect? We present a certified, tolerance-aware algorithm for this intersection problem. The method exploits the fact that likelihood ordering induces halfspace constraints in log-odds coordinates, enabling adaptive geometric partitioning of parameter space and computable lower and upper bounds on $p$-values over each cell. For three categories, this yields an efficient and provably sound algorithm that either certifies intersection, certifies disjointness, or returns an indeterminate result when the decision lies within a prescribed margin. We further show how the approach extends to higher dimensions. The results demonstrate that, despite their irregular geometry, MVCs admit reliable certified decision procedures for core tasks in A/B testing. 
\end{abstract}

\section{Introduction} 

Confidence sets, regions, and intervals are fundamental tools in data science, statistical inference, and machine learning, capturing a range of plausible beliefs of the parameters of a model. For simplicity of computation and analysis, most approaches to construct confidence sets rely on approximation or bounds that are loose in the small sample regime \cite{casella2021statistical, chafai2009confidence, malloy2021ISIT}. While these approaches are often optimal asymptotically, tighter confidence sets in the small sample regime can reduce sample complexity in A/B testing, reinforcement learning algorithms, and other problems in applied data science \cite{malloy2020optimal, jamieson2013finding, malloy2015contamination, malloy2012quickest, malloy2013sample}.

For categorical data, constructing tight confidence sets for the multinomial parameter is a long-studied problem. Recent work \cite{malloy2021ISIT}, building on classical dualities between hypothesis testing and confidence sets \cite{sterne1954some, crow1956confidence, brown1995optimal}, showed that confidence sets defined as level sets of the exact $p$-value attain a minimum-volume optimality property when averaged over empirical outcomes. These \emph{minimum-volume confidence sets} (MVCs) are therefore statistically optimal in a precise sense.

Despite their optimality, MVCs present significant computational challenges. Membership testing for a single parameter value requires evaluating an exact $p$-value defined as a sum over multinomial outcomes, a task that scales poorly with sample size and dimension \cite{resin2020simple}. Moreover, because the $p$-value is a discontinuous function of the parameter, MVCs exhibit highly irregular geometry: they are generally nonconvex, may be disconnected, and do not admit simple geometric characterizations \cite{lin2022geometry}. These features make it difficult to use MVCs directly in even basic inference tasks.

This paper focuses on one such task that is central to A/B testing: given two observed multinomial outcomes, decide whether their corresponding MVCs intersect. An intersection certifies that there exists a parameter value consistent with both outcomes at a given confidence level, while disjointness implies a statistically significant difference between the underlying distributions. Naive approaches based on gridding or continuous relaxations fail to provide guarantees due to the discontinuous and fragmented structure of MVCs.
Figure~\ref{fig:simplex_comparison} illustrates this difficulty in a simple multinomial example with three categories. The asymptotic  chi-square (Likelihood-ratio/Wilks) confidence regions \cite{wilks1938large}, shown in Figure~\ref{fig:chisq_simplex}, are disjoint and would suggest a statistically significant difference between the two empirical outcomes.
In contrast, the exact minimum-volume confidence sets defined via p-value membership, shown in Figure~\ref{fig:exact_simplex}, do intersect, correctly certifying the existence of a parameter value consistent with both observations at the prescribed confidence level. This discrepancy highlights the limitations of asymptotic approximations in the small-sample regime and motivates the need for exact, algorithmically tractable methods for reasoning about intersections of MVCs.

\begin{figure}[t]
  \centering
  \begin{subfigure}[t]{0.48\textwidth}
    \centering
    \includegraphics[width=\linewidth]{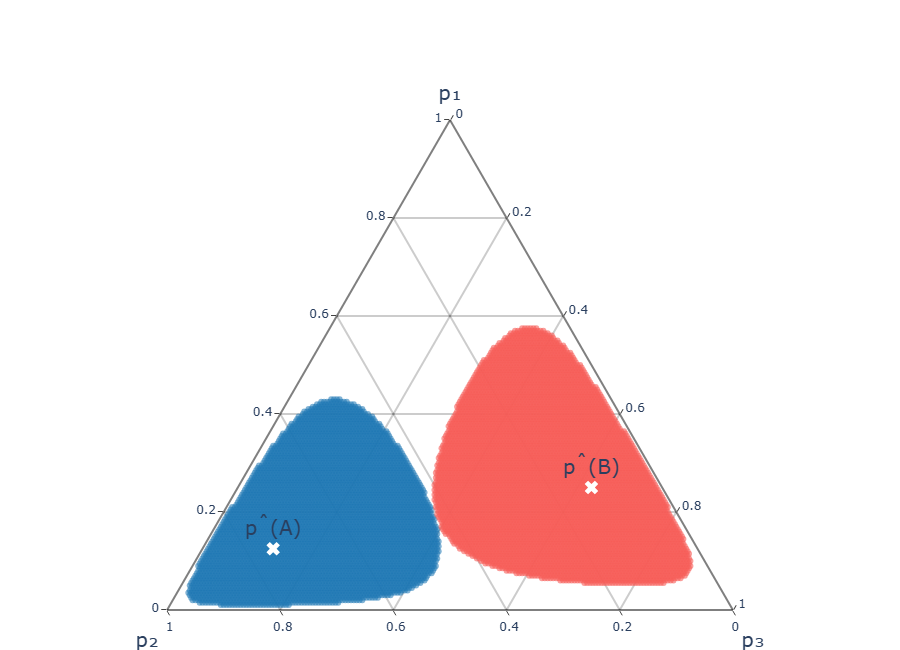}
    \caption{Chi-square (LRT/Wilks) confidence sets.}
    \label{fig:chisq_simplex}
  \end{subfigure}
  \hfill
  \begin{subfigure}[t]{0.48\textwidth}
    \centering
    \includegraphics[width=\linewidth]{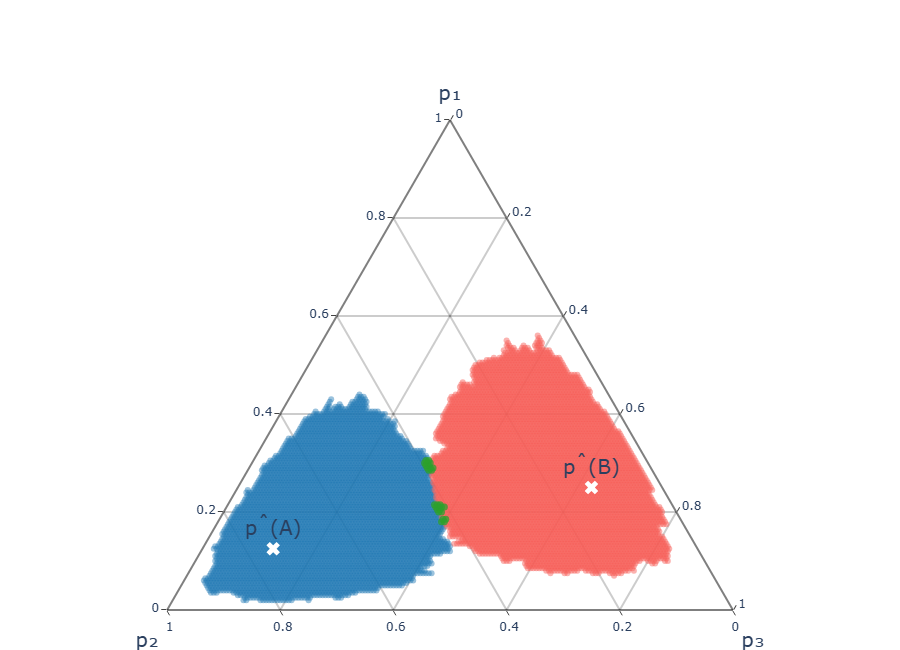}
    \caption{Minimum volume confidence sets.}
    \label{fig:exact_simplex}
  \end{subfigure}

  \caption{
  Comparison of asymptotic and exact confidence sets on the simplex $\Delta_3$.
  The chi-square approximation (left) yields non-intersecting regions,
  whereas the MVCs (right) exhibit an intersection. Setting: $n=8, k =3$ and $\alpha=0.17$, two observed outcomes are $A = [1 , 6,1]$ and  $B = [2,1,5].$
  }
  \label{fig:simplex_comparison}
\end{figure}

Our contribution is a certified decision procedure for this intersection problem. Rather than attempting to compute MVCs explicitly, we exploit a key structural property: likelihood comparisons between multinomial outcomes become linear inequalities in log-odds coordinates. This observation enables us to partition parameter space into geometric cells on which tight lower and upper bounds on the exact $p$-value can be computed. Using adaptive refinement, we obtain an algorithm that either certifies intersection, certifies disjointness, or reports uncertainty when the decision lies within a prescribed tolerance. We present the method in detail for the case of three categories, then extend to general $k$ dimensions.

\section{Notation and Basic Definitions}
Let $X_1, \dots, X_n$ be i.i.d.\ samples of a categorical random variable that takes one of $k$ possible values from a finite number of categories $\mathcal{X} = \{x_1, \dots, x_k\}$. The empirical distribution $\phat$ is the relative proportion of occurrences of each element of $\mathcal{X}$ in $X_1, \dots, X_n$, i.e., $\phat = [\nicefrac{n_1}{n}, \dots, \nicefrac{n_k}{n}]$, where ${n}_i = \sum_{j=1}^n {\1_{\{ X_j = x_i \} }  }$. Let $\Delta_{k,n}$ denote the discrete simplex from $n$ samples over $k$ categories:
\begin{align*}
\Delta_{k,n} \ := \ \left\{ \phat \in \{0, \ \nicefrac{1}{n},\ \nicefrac{2}{n}, \ \dots, \ 1\}^k \ : \  \sum_{i=1}^k \widehat{p}_i =1  \right\},
\end{align*}
and define $m = |\Delta_{k,n}|= { n+k-1 \choose k-1}$. 
Denote the continuous simplex as $\Delta_k = \left\{\bp \in [0,1]^k :  \sum_i p_i =1  \right\}$. We use $\mathcal{P}(\Delta_{k,n})$ to denote the power set of $\Delta_{k,n}$, and $\mathcal{P}(\Delta_k)$ to denote the set of Lebesgue measurable subsets of $\Delta_k$. For any $\mathcal{S} \subset \Delta_{k,n}$ we write $\mathbb{P}_{\boldsymbol{p}}(\mathcal{S})$ as shorthand for  $\mathbb{P}_{\bp} \left(\left\{ X \in \mathcal{X}^n : \phat(X) \in \mathcal{S}\right\} \right)$, where $\mathbb{P}_{\bp}( \cdot)$ denotes the probability measure under the multinomial parameter $\bp \in\Delta_{k}$. Lastly, let $\mathbb{R}^{k+} := \{\bx \in \mathbb{R}^{k}: x_i > 0, i=1,\dots, k \} $.


\begin{defi}(Confidence set)  Let $\mathcal{C}_{\alpha}( \phat): \Delta_{k,n} \rightarrow \mathcal{P}(\Delta_{k})$ be a set valued function that maps an observed empirical distribution $\phat$ to a subset of the $k$-simplex. $\mathcal{C}_{\alpha}( \phat)$   is a \emph{confidence set} at confidence level $1-\alpha$ if the following holds:

\begin{eqnarray} \label{eqn:cr}
\sup_{\bp \in \Delta_{k}} \mathbb{P}_{\bp}  \left( \bp  \not \in \mathcal{C}_{\alpha}(\phat) \right) \ \leq \ \alpha.
\end{eqnarray}
\label{def:cr}
\end{defi}

\begin{defi} ($p$-value) Fix an outcome $\phat$.  The 
$p$-value as a function of the null hypothesis $\bp$ is given by:
\begin{align} \label{eqn:partial}
\rho_{\phat}(\bp) \ = \ \sum_{\qhat \in \Delta_{k,n} : \mathbb{P}_{\bp}(\qhat) \leq \mathbb{P}_{\bp}(\phat)  } \mathbb{P}_{\bp}  \left( \qhat  \right).
\end{align}
\end{defi}
\noindent
For a fixed outcome $\widehat{\bp}$, we write $\rho(\bp)$ for simplicity.

\begin{prop}
\label{def:minvol}
(Minimum volume confidence set (MVCs) \cite{malloy2021ISIT}).  
The MVCs are defined as 
\begin{eqnarray*} 
\mathcal{C}_{\alpha}^\star(\phat) \ := \ \big\{\bp \in\Delta_{k} \ : \ \rho_{\phat}(\bp) \geq \alpha \big\},
\end{eqnarray*}
and satisfy 
\begin{eqnarray*}
\sum_{\phat \in \Delta_{k,n} } \mathrm{vol} \left( \mathcal{C}_{\alpha}^\star(\phat) \right) \leq
\sum_{\phat \in \Delta_{k,n} } \mathrm{vol} \left( \mathcal{C}_{\alpha}(\phat) \right)
\end{eqnarray*}
for any confidence set $\mathcal{C}_{\alpha}( \cdot )$; here $\mathrm{vol}( \cdot ) $ denotes the Lebesgue measure.  A proof can be found in \cite{malloy2021ISIT}. 
\end{prop}

\section{AB-Testing Intersection Problem}\label{sec:ab-problem}

In this paper we focus exclusively on a single computational task motivated by AB testing:
given two observed multinomial outcomes, decide whether their minimum-volume confidence sets
(MVCs) intersect.

Let $\widehat{\bp}^{(A)},\widehat{\bp}^{(B)}\in\Delta_{k,n}$ denote two observed empirical
distributions from $n$ i.i.d.\ categorical samples
over $k$ categories. For a confidence level $1-\alpha$, recall from Section~II that the MVC
associated with $\widehat{\bp}$ is
\[
\cC_\alpha^\star(\widehat{\bp}) \;=\; \{\bp\in\Delta_k:\ \rho_{\widehat{\bp}}(\bp)\ge \alpha\}.
\]
Our goal is to decide whether
\begin{equation}\label{eq:intersection-decision}
\cC_\alpha^\star(\widehat{\bp}^{(A)})\ \cap\ \cC_\alpha^\star(\widehat{\bp}^{(B)})\ \neq\ \emptyset.
\end{equation}

Equivalently, define $\rho_A(\bp):=\rho_{\widehat{\bp}^{(A)}}(\bp)$ and
$\rho_B(\bp):=\rho_{\widehat{\bp}^{(B)}}(\bp)$. Then~\eqref{eq:intersection-decision} holds iff
\[
\max_{\bp\in\Delta_k}\ \min\{\rho_A(\bp),\rho_B(\bp)\} \;\ge\; \alpha.
\]
We emphasize that we do \emph{not} seek to compute or visualize MVC boundaries like \cite{lin2022geometry}; we only seek
a reliable yes/no answer to~\eqref{eq:intersection-decision}.

In finite-precision computation, decisions can be numerically delicate when the optimum is near the threshold $\alpha$. We therefore introduce a \emph{decision margin} $\tau>0$ and target a
certified decision with a buffer around $\alpha$.

\begin{defi}[Robust intersection decision]\label{def:robust-decision}
Fix $\alpha\in(0,1)$ and $\tau>0$. We say that the MVCs \emph{robustly intersect} if there exists
$\bp\in\Delta_k$ such that
\[
\rho_A(\bp)\ge \alpha+\tau
\qquad\text{and}\qquad
\rho_B(\bp)\ge \alpha+\tau.
\]
We say that the MVCs are \emph{robustly disjoint} if for all $\bp\in\Delta_k$,
\[
\min\{\rho_A(\bp),\rho_B(\bp)\}\le \alpha-\tau.
\]
\end{defi}

Our algorithm will either return 
\begin{itemize}
    \item \textsf{INTERSECT} (certifying robust intersection)
    \item \textsf{DISJOINT} (certifying robust disjointness), or
    \item  \textsf{UNCERTAIN} when the instance cannot be resolved within the chosen tolerances.
\end{itemize}
 In addition to $\tau$, the algorithm takes a \emph{geometric refinement tolerance} $\varepsilon>0$ controlling the maximum cell diameter in an
adaptive partition.

\section{Certified Intersection Test for $k=3$}\label{sec:k3}

We first present the method cleanly for $k=3$, where the simplex is two-dimensional and the
algorithm is especially effective. Section~\ref{sec:general-k} then generalizes the approach to
arbitrary $k$.

\subsection{Log-odds coordinates and outcome probabilities}\label{sec:uv}

For $k=3$ we write $\bp=(p_1,p_2,p_3)\in\Delta_3^\circ$, , all coordinates are
strictly positive so that log-odds coordinates are finite. Define log-odds coordinates
$(u,v)\in\R^2$ by
\[
u=\log\frac{p_1}{p_3},
\qquad
v=\log\frac{p_2}{p_3},
\]
with inverse map
\begin{equation}\label{eq:softmax-3}
p_1(u,v)=\frac{e^u}{1+e^u+e^v},\quad
p_2(u,v)=\frac{e^v}{1+e^u+e^v},\quad
p_3(u,v)=\frac{1}{1+e^u+e^v}.
\end{equation}
For any count vector $r=(r_1,r_2,r_3)\in\N^3$ with $r_1+r_2+r_3=n$, define the multinomial
coefficient
\[
\kappa(r)=\frac{n!}{r_1!\,r_2!\,r_3!}.
\]
Then the probability of observing $r$ under parameter $\bp(u,v)$ is
\[
\PP_{\bp(u,v)}(r)\;=\;\kappa(r)\,p_1(u,v)^{r_1}\,p_2(u,v)^{r_2}\,p_3(u,v)^{r_3}.
\]
Taking logs and using~\eqref{eq:softmax-3} gives
\begin{equation}\label{eq:logprob-3}
\log \PP_{\bp(u,v)}(r)
=
\log\kappa(r)\;+\; r_1 u + r_2 v\;-\;n\log(1+e^u+e^v).
\end{equation}

In the presence of zero counts, the intersection decision is understood as allowing $\bp$ to lie on the boundary of $\Delta_k$. As shown in Section~\ref{sec:boundary}, this can be reduced to a finite collection of lower-dimensional multinomial distributions on the faces of the simplex.

\subsection{Likelihood-ordering halfspaces}\label{sec:halfspaces}

Fix an observed outcome $\hat r=(\hat r_1,\hat r_2,\hat r_3)$ (equivalently
$\widehat{\bp}=\hat r/n\in\Delta_{3,n}$). Recall that the p-value is
\[
\rho_{\widehat{\bp}}(\bp)
=
\sum_{r:\ \PP_{\bp}(r)\le \PP_{\bp}(\hat r)} \PP_{\bp}(r).
\]
A key fact is that the likelihood-ordering comparison $\PP_{\bp(u,v)}(r)\le \PP_{\bp(u,v)}(\hat r)$
is \emph{linear} in $(u,v)$.

\begin{lemma}[Halfspace form in $(u,v)$]\label{lem:halfspace-3}
For any $r$ and fixed $\hat r$, the inequality $\PP_{\bp(u,v)}(r)\le \PP_{\bp(u,v)}(\hat r)$ is
equivalent to
\begin{equation}\label{eq:halfspace-ineq-3}
(r_1-\hat r_1)u + (r_2-\hat r_2)v \;\le\; \log\kappa(\hat r)-\log\kappa(r).
\end{equation}
\end{lemma}
\begin{proof}
Subtract~\eqref{eq:logprob-3} for $r$ and for $\hat r$. The shared term $-n\log(1+e^u+e^v)$ cancels.
Exponentiating yields~\eqref{eq:halfspace-ineq-3}.
\end{proof}

Thus, for fixed $\hat r$, each $r$ defines a halfspace in $(u,v)$, and the set of $(u,v)$ where
$r$ is in the p-value tail (i.e.\ has probability no larger than $\hat r$) is a halfspace boundary.

\subsection{A compact search domain}\label{sec:domain}

The coordinates $(u,v)$ range over $\R^2$, so we restrict the search to a compact region without
losing any robust decision. The following crude bound is sufficient.

\begin{lemma}[A universal upper bound on the p-value]\label{lem:mpx-bound}
Let $m=|\Delta_{3,n}|$. For any observed $\hat r$ and any $\bp\in\Delta_3$,
\[
\rho_{\widehat{\bp}}(\bp)\ \le\ m\,\PP_{\bp}(\hat r).
\]
\end{lemma}
\begin{proof}
The tail set in the p-value definition contains at most $m$ outcomes; each term in the tail sum
is $\le \PP_{\bp}(\hat r)$ by definition. Summing yields the claim. A detailed proof is provided in Appendix~\ref{lem:mpx-bound-detailed}.
\end{proof}

By Lemma~\ref{lem:mpx-bound}, if $\PP_{\bp}(\hat r)<(\alpha-\tau)/m$ then
$\rho_{\widehat{\bp}}(\bp)<\alpha-\tau$. Therefore any $\bp$ that could witness robust intersection
(Definition~\ref{def:robust-decision}) must satisfy \emph{both}
\[
\PP_{\bp}(\hat r^{(A)})\ge (\alpha-\tau)/m
\qquad\text{and}\qquad
\PP_{\bp}(\hat r^{(B)})\ge (\alpha-\tau)/m.
\]
In $(u,v)$-coordinates, each constraint is a superlevel set of the concave function
$(u,v)\mapsto \log\PP_{\bp(u,v)}(\cdot)$, hence convex and bounded. Let $\cW\subset\R^2$ be any
compact convex set containing the intersection of these two superlevel sets. In practice, a
simple bounding box around that intersection suffices; the algorithm below only requires that
$\cW$ contain all potential robust witnesses.

\subsection{Cellwise lower bounds on multinomial probabilities}\label{sec:prob-bounds}

We partition $\cW$ into triangles. Let $T\subset\R^2$ be a triangle with vertices
$w^{(1)},w^{(2)},w^{(3)}\in\cW$, where $w^{(j)}=(u^{(j)},v^{(j)})$.

For any outcome $r$, define $f_r(u,v):=\log\PP_{\bp(u,v)}(r)$. From~\eqref{eq:logprob-3},
$f_r$ is concave in $(u,v)$ because it is affine minus $n\log(1+e^u+e^v)$ (and $\log(1+e^u+e^v)$
is convex). Concavity implies a simple, computable lower bound on $\PP_{\bp(u,v)}(r)$ over $T$.

\begin{lemma}[Vertex minimum for concave log-probabilities]\label{lem:vertex-min}
For any outcome $r$ and any triangle $T$ with vertices $\{w^{(j)}\}_{j=1}^3$,
\[
\min_{(u,v)\in T}\ f_r(u,v)\ =\ \min_{j\in\{1,2,3\}} f_r(w^{(j)}).
\]
Consequently, the quantity
\[
\underline{P}_T(r)\ :=\ \min_{j\in\{1,2,3\}} \PP_{\bp(w^{(j)})}(r)
\]
satisfies $\PP_{\bp(u,v)}(r)\ge \underline{P}_T(r)$ for all $(u,v)\in T$.
\end{lemma}
\begin{proof}
Any $(u,v)\in T$ is a convex combination of the vertices. Concavity gives
$f_r(u,v)\ge \sum_j \lambda_j f_r(w^{(j)})\ge \min_j f_r(w^{(j)})$. Please refer to Appendix~\ref{app:vertex-min-proof} for details.
\end{proof}

\subsection{Certified p-value bounds on a triangle}\label{sec:pval-bounds}

Fix an observed outcome $\hat r$. For each $r$, define the linear function
\[
g_{r;\hat r}(u,v)
:=
(r_1-\hat r_1)u + (r_2-\hat r_2)v - \big(\log\kappa(\hat r)-\log\kappa(r)\big).
\]
By Lemma~\ref{lem:halfspace-3}, $r$ is in the p-value tail at $(u,v)$ iff $g_{r;\hat r}(u,v)\le 0$.
Since $g_{r;\hat r}$ is linear, membership can be certified over a triangle by checking vertices.

Define the \emph{definitely-in-tail} $\cI_T(\hat r)$ and \emph{definitely-out-of-tail} $\cO_T(\hat r)$ index sets for triangle $T$:
\[
\cI_T(\hat r):=\{r:\ g_{r;\hat r}(w^{(j)})\le 0\ \text{for all }j=1,2,3\},
\qquad
\cO_T(\hat r):=\{r:\ g_{r;\hat r}(w^{(j)})> 0\ \text{for all }j=1,2,3\}.
\]
All remaining outcomes are \emph{ambiguous} on $T$.

\begin{prop}[Certified p-value interval on a triangle]\label{prop:pval-interval}
Let $\hat r$ be fixed and let $T$ be a triangle. Define
\begin{equation}\label{eq:pval-LB-UB}
\underline{\rho}_T(\hat r)
:=
\sum_{r\in\cI_T(\hat r)} \underline{P}_T(r),
\qquad
\overline{\rho}_T(\hat r)
:=
1-\sum_{r\in\cO_T(\hat r)} \underline{P}_T(r),
\end{equation}
where $\underline{P}_T(r)$ is from Lemma~\ref{lem:vertex-min}.
Then for all $(u,v)\in T$,
\[
\underline{\rho}_T(\hat r)\ \le\ \rho_{\widehat{\bp}}(\bp(u,v))\ \le\ \overline{\rho}_T(\hat r).
\]
\end{prop}
\begin{proof}
If $r\in\cI_T(\hat r)$ then $g_{r;\hat r}\le 0$ at all vertices, hence everywhere on $T$, so $r$
belongs to the tail for all $(u,v)\in T$. Therefore the p-value tail sum includes all
$r\in\cI_T(\hat r)$, and Lemma~\ref{lem:vertex-min} lower-bounds each probability by
$\underline{P}_T(r)$, yielding the lower bound.

If $r\in\cO_T(\hat r)$ then $g_{r;\hat r}>0$ at all vertices, hence everywhere on $T$, so $r$ is
excluded from the tail for all $(u,v)\in T$. Thus the p-value is at most the total probability
mass of outcomes not in $\cO_T(\hat r)$. Using that $\PP_{\bp(u,v)}(r)\ge \underline{P}_T(r)$ for
each $r\in\cO_T(\hat r)$ and $\sum_r \PP_{\bp(u,v)}(r)=1$ yields the upper bound in~\eqref{eq:pval-LB-UB}.
\end{proof}

\subsection{Triangle refinement and the certified intersection test}\label{sec:algo-k3}

For AB testing we require both p-values to exceed $\alpha$. On a triangle $T$, define
\[
\underline{\rho}_{A,T}:=\underline{\rho}_T(\hat r^{(A)}),\quad
\overline{\rho}_{A,T}:=\overline{\rho}_T(\hat r^{(A)}),
\qquad
\underline{\rho}_{B,T}:=\underline{\rho}_T(\hat r^{(B)}),\quad
\overline{\rho}_{B,T}:=\overline{\rho}_T(\hat r^{(B)}),
\]
and the corresponding bounds for the minimum:
\[
\underline{M}_T:=\min\{\underline{\rho}_{A,T},\underline{\rho}_{B,T}\},
\qquad
\overline{M}_T:=\min\{\overline{\rho}_{A,T},\overline{\rho}_{B,T}\}.
\]
Then:
(i) if $\underline{M}_T\ge \alpha+\tau$, $T$ certifies robust intersection;
(ii) if $\overline{M}_T<\alpha-\tau$, $T$ can be discarded (no robust witness in $T$);
(iii) otherwise, $T$ is refined unless it is already sufficiently small.

We measure triangle size by its diameter in $(u,v)$:
\[
\diam(T):=\max_{i,j\in\{1,2,3\}}\|w^{(i)}-w^{(j)}\|_2.
\]
We refine by longest-edge bisection (any shape-regular refinement works). The discussion so far enable us to design the certified MVC intersection test Algorithm~\ref{alg:k3} for $k = 3$.

\begin{algorithm}[t]
\caption{Certified MVC intersection test for $k=3$}\label{alg:k3}
\begin{algorithmic}[1]
\Require Observed outcomes $\hat r^{(A)},\hat r^{(B)}\in\N^3$ with sum $n$, level $\alpha\in(0,1)$,
decision margin $\tau>0$, refinement tolerance $\varepsilon>0$, compact search domain $\cW\subset\R^2$.
\Ensure \textsf{INTERSECT}, \textsf{DISJOINT}, or \textsf{UNCERTAIN}.
\State Initialize a triangulation $\cT$ of $\cW$ (e.g.\ split a bounding box into two triangles).
\State Initialize a queue $\cQ\gets \cT$.
\While{$\cQ$ is not empty}
    \State Pop a triangle $T$ from $\cQ$.
    \State Compute $(\underline{\rho}_{A,T},\overline{\rho}_{A,T})$ and $(\underline{\rho}_{B,T},\overline{\rho}_{B,T})$ via Proposition~\ref{prop:pval-interval}.
    \State Set $\underline{M}_T\gets \min\{\underline{\rho}_{A,T},\underline{\rho}_{B,T}\}$ and $\overline{M}_T\gets \min\{\overline{\rho}_{A,T},\overline{\rho}_{B,T}\}$.
    \If{$\underline{M}_T \ge \alpha+\tau$}
        \State \Return \textsf{INTERSECT}.
    \ElsIf{$\overline{M}_T < \alpha-\tau$}
        \State \textbf{continue prune $T$} 
    \ElsIf{$\diam(T)\le \varepsilon$}
        \State Mark $T$ as unresolved.
    \Else
        \State Bisect $T$ into two triangles $T_1,T_2$ (e.g.\ longest-edge bisection).
        \State Push $T_1,T_2$ onto $\cQ$.
    \EndIf
\EndWhile
\If{no unresolved triangles exist}
    \State \Return \textsf{DISJOINT}.
\Else
    \State \Return \textsf{UNCERTAIN}.
\EndIf
\end{algorithmic}
\end{algorithm}

\subsection{Correctness guarantee}\label{sec:correctness-k3}


\begin{thm}[Soundness of Algorithm~\ref{alg:k3}]\label{thm:soundness-k3}
Let $m=|\Delta_{3,n}|$ and define
\[
S_A := \{(u,v)\in\R^2 : \PP_{\bp(u,v)}(\hat r^{(A)})\ge (\alpha-\tau)/m\},\quad
S_B := \{(u,v)\in\R^2 : \PP_{\bp(u,v)}(\hat r^{(B)})\ge (\alpha-\tau)/m\}.
\]
Assume the compact search domain $\cW$ satisfies $\cW \supseteq S_A\cap S_B$.
Then Algorithm~\ref{alg:k3} is sound in the following sense:
\begin{itemize}
\item If it returns \textsf{INTERSECT}, then the MVCs robustly intersect (Definition~\ref{def:robust-decision}).
\item If it returns \textsf{DISJOINT}, then the MVCs are robustly disjoint (Definition~\ref{def:robust-decision}).
\end{itemize}
If it returns \textsf{UNCERTAIN}, then the true decision is within a $\tau$-neighborhood of the
threshold on at least one triangle of diameter at most $\varepsilon$.
\end{thm}

\begin{proof}
On any triangle $T$, Proposition~\ref{prop:pval-interval} bounds each p-value over all points in $T$,
hence $\underline{M}_T\le \min\{\rho_A,\rho_B\}\le \overline{M}_T$ on $T$.

If $\underline{M}_T\ge \alpha+\tau$ for some processed $T$, then for all $\bp\in T$,
$\min\{\rho_A(\bp),\rho_B(\bp)\}\ge \alpha+\tau$, so a robust witness exists and the MVCs robustly intersect. If the queue empties with no unresolved triangles, then for every triangle in the (refined) triangulation of $\cW$ we have $\overline{M}_T<\alpha-\tau$, hence $\min\{\rho_A(\bp),\rho_B(\bp)\}<\alpha-\tau$ for all $\bp\in\cW$.
Now consider any $\bp\in\Delta_3$ with coordinates $(u,v)\notin \cW$. Since $\cW \supseteq S_A\cap S_B$, we have $(u,v)\notin S_A$ or $(u,v)\notin S_B$.
If $(u,v)\notin S_A$, then $\PP_{\bp}(\hat r^{(A)})<(\alpha-\tau)/m$, and by Lemma~\ref{lem:mpx-bound},
$\rho_A(\bp)<\alpha-\tau$. Similarly for $B$. In either case, $\min\{\rho_A(\bp),\rho_B(\bp)\}\le \alpha-\tau$ holds for all $\bp\notin\cW$. Combining with the bound on $\cW$ proves robust disjointness over all $\Delta_3$.
\end{proof}

\subsection{Computational complexity}\label{sec:remarks-k3}

The naive per-triangle cost is $O(m)$ outcomes times a small constant, where
$m=|\Delta_{3,n}|=\binom{n+2}{2}=O(n^2)$. In practice, substantial speedups are obtained by:
(i) precomputing $\log\kappa(r)$ for all $r$ once; (ii) caching $\PP_{\bp(w^{(j)})}(r)$ across
triangles that share vertices; and (iii) early pruning using $\overline{M}_T<\alpha-\tau$.

\section{Extension to general $k$}\label{sec:general-k}

We now generalize the construction to arbitrary $k$, highlighting what changes and what remains
identical.

\subsection{Log-odds coordinates in dimension $k-1$}\label{sec:logodds-k}

For $\bp\in\Delta_k^\circ$, define $u\in\R^{k-1}$ by
\[
u_i=\log\frac{p_i}{p_k},\qquad i=1,\ldots,k-1,
\]
with inverse map
\[
p_i(u)=\frac{e^{u_i}}{1+\sum_{j=1}^{k-1}e^{u_j}},\ i\le k-1,
\qquad
p_k(u)=\frac{1}{1+\sum_{j=1}^{k-1}e^{u_j}}.
\]
For any outcome $r=(r_1,\ldots,r_k)$ with $\sum_i r_i=n$, let
$\kappa(r)=\frac{n!}{\prod_{i=1}^k r_i!}$. Then
\[
\log \PP_{\bp(u)}(r)
=
\log\kappa(r)\;+\;\sum_{i=1}^{k-1} r_i u_i\;-\;n\log\!\Big(1+\sum_{j=1}^{k-1}e^{u_j}\Big).
\]

\subsection{Halfspace structure and simplex cells}\label{sec:halfspaces-k}

Fix an observed $\hat r$. Exactly as in Lemma~\ref{lem:halfspace-3}, for any $r$,
\[
\PP_{\bp(u)}(r)\le \PP_{\bp(u)}(\hat r)
\quad\Longleftrightarrow\quad
\sum_{i=1}^{k-1}(r_i-\hat r_i)u_i \;\le\; \log\kappa(\hat r)-\log\kappa(r),
\]
which is a halfspace in $\R^{k-1}$.

To keep vertex checks inexpensive in higher dimension, we recommend partitioning the compact
search region $\cW\subset\R^{k-1}$ into \emph{simplices}. A simplex in dimension $d=k-1$ has only
$d+1=k$ vertices (as opposed to $2^d$ vertices for a hypercube).

\subsection{Certified p-value bounds on a simplex cell}\label{sec:bounds-k}

Let $S\subset\R^{k-1}$ be a simplex with vertices $\{w^{(j)}\}_{j=1}^{k}$. Define the linear
functions
\[
g_{r;\hat r}(u)
:=
\sum_{i=1}^{k-1}(r_i-\hat r_i)u_i - \big(\log\kappa(\hat r)-\log\kappa(r)\big).
\]
As before, define
\[
\cI_S(\hat r):=\{r:\ g_{r;\hat r}(w^{(j)})\le 0\ \text{for all }j\},
\qquad
\cO_S(\hat r):=\{r:\ g_{r;\hat r}(w^{(j)})> 0\ \text{for all }j\}.
\]
The function $u\mapsto \log\PP_{\bp(u)}(r)$ is concave in $\R^{k-1}$, so the vertex-minimum
lower bound generalizes directly:
\[
\underline{P}_S(r):=\min_{j\in\{1,\ldots,k\}}\PP_{\bp(w^{(j)})}(r)
\quad\Rightarrow\quad
\PP_{\bp(u)}(r)\ge \underline{P}_S(r)\ \ \forall u\in S.
\]
Hence the p-value bounds in Proposition~\ref{prop:pval-interval} hold verbatim with $T$ replaced
by $S$:
\[
\underline{\rho}_S(\hat r)=\sum_{r\in\cI_S(\hat r)}\underline{P}_S(r),
\qquad
\overline{\rho}_S(\hat r)=1-\sum_{r\in\cO_S(\hat r)}\underline{P}_S(r).
\]

\subsection{Algorithm and scaling discussion}\label{sec:scaling-k}

Algorithm~\ref{alg:k3} generalizes by replacing triangles with $(k-1)$-simplices and using a
shape-regular refinement rule (e.g.\ newest-vertex bisection). The \textsf{INTERSECT}/\textsf{DISJOINT}
soundness remains unchanged.

The main difference for $k>3$ is computational scaling:
$m=|\Delta_{k,n}|=\binom{n+k-1}{k-1}$ grows as $n^{k-1}$ for fixed $k$, and adaptive refinement may
require many cells in dimension $k-1$ in worst cases. Nonetheless, for modest $k$ (as is common in
categorical AB testing), the method remains practical because
(i) simplex cells use only $k$ vertices;
(ii) vertex-probability evaluations can be cached aggressively; and
(iii) many cells prune early via $\overline{M}_S<\alpha-\tau$.

\subsection{Handling zero-count categories via face decomposition}
\label{sec:boundary}

The development so far implicitly assumes $\bp\in\Delta_k^\circ$, i.e., $p_i>0$ for all
$i$, so that log-odds coordinates are finite. In practice, observed multinomial outcomes
often contain zero counts, and any correct intersection test must account for boundary
faces of the simplex. We show that this can be handled cleanly by a finite decomposition
over lower-dimensional faces.

Let $\hat r^{(A)},\hat r^{(B)}\in\N^k$ be the observed outcomes. Define the index sets
\[
S_A := \{i:\hat r^{(A)}_i>0\},\qquad
S_B := \{i:\hat r^{(B)}_i>0\},\qquad
S_0 := S_A\cup S_B.
\]

\begin{lemma}[Zero-count exclusion]
\label{lem:zero-exclusion}
If $\hat r_i>0$ for some outcome $\hat r$ and $p_i=0$, then
$\rho_{\hat r}(\bp)=0$.
\end{lemma}

\begin{proof}
If $p_i=0$ and $\hat r_i>0$, then $\PP_{\bp}(\hat r)=0$. The p-value tail
$\{r:\PP_{\bp}(r)\le \PP_{\bp}(\hat r)\}$ contains only outcomes of probability zero,
hence the p-value sum equals zero.
\end{proof}

Lemma~\ref{lem:zero-exclusion} implies that any parameter $\bp$ satisfying
$\rho_A(\bp)\ge\alpha$ and $\rho_B(\bp)\ge\alpha$ must satisfy $p_i>0$ for all $i\in S_0$.
Consequently, only categories with zero count in \emph{both} outcomes may be set to zero in a potential intersection witness. We provide an example of such case in Appendix~\ref{app:example-zero-count}.

Let
\[
Z := \{1,\ldots,k\}\setminus S_0
\]
denote the set of jointly-zero categories. For any subset $T\subseteq Z$, define the face
\[
\Delta_k^{(T)} := \{\bp\in\Delta_k:\ p_i=0\ \forall i\in T\}.
\]
On this face, the multinomial distribution reduces to a $(k-|T|)$-category multinomial
over the remaining coordinates, with the same sample size $n$. The corresponding reduced
outcomes are obtained by deleting the coordinates in $T$ from $\hat r^{(A)}$ and
$\hat r^{(B)}$.

Therefore, deciding whether the MVCs intersect reduces to checking robust intersection on
each face $\Delta_k^{(T)}$, $T\subseteq Z$, in the reduced dimension. We provide an example of such case in Appendix~\ref{app:example-collapsing}. A general algorithm that extends to $k$-dimension with zero-count handling can be find in Algorithm~\ref{alg:faces}.

\section{Summary}
This paper studied a fundamental computational problem associated with minimum-volume confidence sets for the multinomial parameter: deciding whether the confidence sets corresponding to two observed outcomes intersect. Rather than attempting to compute or visualize MVCs directly, we developed a certified decision procedure that exploits the geometric structure induced by likelihood ordering in log-odds coordinates.

For the case of three categories, we presented an adaptive partitioning algorithm that computes rigorous lower and upper bounds on the exact $p$-value over each cell, yielding provably sound certificates of intersection or disjointness up to a user-specified tolerance. The method naturally accommodates the discontinuity and nonconvexity of MVCs by reasoning over regions on which the ordering of likelihoods is fixed. We further showed how the approach extends to higher dimensions and how boundary cases arising from zero-count categories can be handled cleanly by decomposing the simplex into a finite collection of lower-dimensional faces.

\bibliography{main}
\bibliographystyle{IEEEtran}

\section*{Appendix}

\subsection{Detailed Proof of Lemma~\ref{lem:mpx-bound}}
\noindent \textbf{Restatement of Lemma~\ref{lem:mpx-bound}}.
\label{lem:mpx-bound-detailed}
Let $k\ge 2$ and $n\ge 1$, and let $\Delta_{k,n}$ denote the set of all empirical
distributions (equivalently, count vectors divided by $n$) arising from $n$ i.i.d.\
samples over $k$ categories. Let $m:=|\Delta_{k,n}|=\binom{n+k-1}{k-1}$. Fix any observed
outcome $\hat r$ (or equivalently $\widehat{\bp}=\hat r/n\in\Delta_{k,n}$). Then for every
multinomial parameter $\bp\in\Delta_k$,
\[
\rho_{\widehat{\bp}}(\bp)\ \le\ m\,\PP_{\bp}(\hat r).
\]
Equivalently, writing $\rho_{\hat r}(\bp)$ for the same p-value,
\[
\rho_{\hat r}(\bp)\ \le\ m\,\PP_{\bp}(\hat r).
\]

\begin{proof}
We begin by unpacking the definition of the (exact) p-value used throughout the paper.
Fix $\hat r$ and $\bp\in\Delta_k$. Recall that the p-value associated with $\hat r$ under
the null hypothesis $\bp$ is
\begin{equation}
\label{eq:pval-def-detailed}
\rho_{\hat r}(\bp)
\ :=\
\sum_{r \in \mathcal{T}_{\hat r}(\bp)} \PP_{\bp}(r),
\qquad
\text{where}\quad
\mathcal{T}_{\hat r}(\bp)
:=
\Big\{
r\in\Delta_{k,n}:\ \PP_{\bp}(r)\le \PP_{\bp}(\hat r)
\Big\}.
\end{equation}
Here $\Delta_{k,n}$ indexes all possible empirical outcomes from $n$ draws (equivalently,
all count vectors $r=(r_1,\ldots,r_k)\in\N^k$ with $\sum_i r_i=n$), and $\PP_{\bp}(r)$ is
the multinomial probability of observing that outcome under $\bp$.

The proof is a direct counting-and-bounding argument.

\medskip
\noindent We first show that the tail set contains at most $m$ outcomes.
By definition, $\mathcal{T}_{\hat r}(\bp)\subseteq \Delta_{k,n}$. Since
$|\Delta_{k,n}|=m$, we have the trivial cardinality bound
\begin{equation}
\label{eq:tail-card}
|\mathcal{T}_{\hat r}(\bp)|\ \le\ |\Delta_{k,n}|\ =\ m.
\end{equation}

\medskip
\noindent Then we show that each summand in the p-value is bounded by $\PP_{\bp}(\hat r)$.
For any $r\in\mathcal{T}_{\hat r}(\bp)$, the defining property of $\mathcal{T}_{\hat r}(\bp)$
gives
\begin{equation}
\label{eq:term-bound}
\PP_{\bp}(r)\ \le\ \PP_{\bp}(\hat r).
\end{equation}
This holds for every $r$ included in the p-value sum~\eqref{eq:pval-def-detailed}.

\medskip
\noindent No we bound the total tail probability mass.
Using~\eqref{eq:pval-def-detailed},~\eqref{eq:term-bound}, and the fact that there are
$|\mathcal{T}_{\hat r}(\bp)|$ terms in the sum, we obtain
\begin{align}
\rho_{\hat r}(\bp)
\ &=\
\sum_{r\in\mathcal{T}_{\hat r}(\bp)} \PP_{\bp}(r)
\ \le\
\sum_{r\in\mathcal{T}_{\hat r}(\bp)} \PP_{\bp}(\hat r)
\ =\
|\mathcal{T}_{\hat r}(\bp)|\,\PP_{\bp}(\hat r).
\label{eq:pval-card-times}
\end{align}
Finally apply the cardinality bound~\eqref{eq:tail-card} to~\eqref{eq:pval-card-times}:
\[
\rho_{\hat r}(\bp)
\ \le\
m\,\PP_{\bp}(\hat r),
\]
which is exactly the desired inequality.
\end{proof}

\subsection{Geometric interpretation of the superlevel-set restriction}
\label{app:superlevel-geometry}

We now justify the geometric statement used following Lemma~\ref{lem:mpx-bound}, namely
that the constraint
\[
\PP_{\bp}(\hat r)\ \ge\ \frac{\alpha-\tau}{m}
\]
defines a convex and bounded region in log-odds coordinates.

\medskip
We start by expressing the constraint as a superlevel set.
Fix an observed outcome $\hat r$ and consider the function
\[
f(u,v)
\ :=\
\log \PP_{\bp(u,v)}(\hat r),
\]
where $(u,v)\in\mathbb{R}^2$ are the log-odds coordinates defined by
\[
p_1=\frac{e^u}{1+e^u+e^v},\qquad
p_2=\frac{e^v}{1+e^u+e^v},\qquad
p_3=\frac{1}{1+e^u+e^v}.
\]
The inequality
\[
\PP_{\bp(u,v)}(\hat r)\ \ge\ \frac{\alpha-\tau}{m}
\]
is equivalent, after taking logarithms on both sides, to
\begin{equation}
\label{eq:superlevel}
f(u,v)\ \ge\ \log(\alpha-\tau)-\log m.
\end{equation}
The set of points $(u,v)$ satisfying~\eqref{eq:superlevel} is therefore the
\emph{superlevel set} of the function $f$, i.e.,
\[
\mathcal{S}
\ :=\
\{(u,v)\in\mathbb{R}^2:\ f(u,v)\ge c\},
\qquad
c:=\log(\alpha-\tau)-\log m.
\]

\medskip
\noindent We now show the concavity of the log-likelihood function.
Recall from~\eqref{eq:logprob-3} in the main text that
\[
f(u,v)
=
\log\kappa(\hat r)
+\hat r_1 u+\hat r_2 v
-n\log(1+e^u+e^v).
\]
The first three terms are affine functions of $(u,v)$.
The final term $-\log(1+e^u+e^v)$ is the negative of a convex function, since
$\log(1+e^u+e^v)$ is convex on $\mathbb{R}^2$.
Therefore $f(u,v)$ is a concave function on $\mathbb{R}^2$.
A standard result from convex analysis states that if $f:\mathbb{R}^d\to\mathbb{R}$
is concave, then for any constant $c\in\mathbb{R}$, the superlevel set
\[
\{x:\ f(x)\ge c\}
\]
is a convex subset of $\mathbb{R}^d$.
Applying this result to $f(u,v)$ shows that the set $\mathcal{S}$ defined above
is convex.

\medskip
\noindent 
We now show that $\mathcal{S}$ is bounded.
As $\|(u,v)\|\to\infty$, at least one of $u$ or $v$ tends to $+\infty$ or $-\infty$.
In all such cases, the term $-n\log(1+e^u+e^v)$ dominates and tends to $-\infty$,
while the linear terms $\hat r_1 u+\hat r_2 v$ grow at most linearly.
Consequently,
\[
\lim_{\|(u,v)\|\to\infty} f(u,v) = -\infty.
\]
This implies that for any finite constant $c$, the inequality $f(u,v)\ge c$
can only hold within a bounded region of $\mathbb{R}^2$.
Hence the superlevel set $\mathcal{S}$ is bounded.

\medskip
Therefore, the constraint $\PP_{\bp}(\hat r)\ge(\alpha-\tau)/m$ defines, in $(u,v)$-coordinates,
a bounded and convex subset of $\mathbb{R}^2$.
Applying this argument separately to $\hat r^{(A)}$ and $\hat r^{(B)}$, we conclude
that the intersection of the two corresponding superlevel sets is also convex and
bounded. Any compact convex set $\mathcal{W}$ containing this intersection may
therefore be used as the search domain in Algorithm~\ref{alg:k3}.

\subsection{Detailed Proof of Lemma~\ref{lem:vertex-min}}
\label{app:vertex-min-proof}

\noindent \textbf{Restatement of Lemma~\ref{lem:vertex-min}}.
\label{lem:vertex-min-detailed}
Fix $k=3$ and $n\ge 1$. For any outcome $r=(r_1,r_2,r_3)\in\N^3$ with $r_1+r_2+r_3=n$, define
\[
f_r(u,v)\ :=\ \log \PP_{\bp(u,v)}(r),
\]
where $\bp(u,v)$ is given by the softmax map~\eqref{eq:softmax-3}. Let $T\subset\mathbb{R}^2$
be any (closed) triangle with vertices $w^{(1)},w^{(2)},w^{(3)}\in\mathbb{R}^2$. Then
\[
\min_{(u,v)\in T} f_r(u,v)\ =\ \min_{j\in\{1,2,3\}} f_r\!\left(w^{(j)}\right).
\]
Consequently, the quantity
\[
\underline{P}_T(r)\ :=\ \min_{j\in\{1,2,3\}} \PP_{\bp(w^{(j)})}(r)
\]
satisfies $\PP_{\bp(u,v)}(r)\ge \underline{P}_T(r)$ for all $(u,v)\in T$.

\begin{proof}
We proceed in three steps. First we verify concavity of $f_r$ on $\mathbb{R}^2$. Second we
use a general fact about concave functions on convex polytopes: the minimum over a polytope is attained at an extreme point (vertex). Third we specialize to triangles and conclude the probability lower bound.

\medskip
\noindent We first show that $f_r(u,v)$ is concave in $(u,v)$.
Recall the multinomial probability in log-odds coordinates (cf.~\eqref{eq:logprob-3}):
\begin{equation}
\label{eq:fr-form}
f_r(u,v)
=
\log \kappa(r)\;+\; r_1 u + r_2 v\;-\;n\log(1+e^u+e^v),
\end{equation}
where $\kappa(r)=\frac{n!}{r_1!\,r_2!\,r_3!}$ is constant with respect to $(u,v)$.
The first three terms on the right-hand side of~\eqref{eq:fr-form} are affine functions of
$(u,v)$, hence are both convex and concave. The remaining term
\[
(u,v)\ \mapsto\ \log(1+e^u+e^v)
\]
is convex on $\mathbb{R}^2$ (this is the standard ``log-sum-exp'' function). Multiplying a
convex function by $-n<0$ yields a concave function. Therefore the sum of the affine terms
and the concave term $-n\log(1+e^u+e^v)$ is concave, proving that $f_r$ is concave on
$\mathbb{R}^2$.

\medskip
\noindent Then we show that the minimum of a concave function over a triangle occurs at a vertex.
The triangle $T=\mathrm{conv}\{w^{(1)},w^{(2)},w^{(3)}\}$ is a compact convex set. We will show
that for any $(u,v)\in T$,
\begin{equation}
\label{eq:vertex-lower-bound}
f_r(u,v)\ \ge\ \min_{j\in\{1,2,3\}} f_r\!\left(w^{(j)}\right).
\end{equation}
Since $T$ is compact and $f_r$ is continuous (as a sum of continuous functions), $f_r$
attains its minimum on $T$ by the extreme value theorem. Inequality~\eqref{eq:vertex-lower-bound}
then implies that the minimum value cannot be smaller than the minimum over the vertices,
hence the minimum must be achieved at (at least) one vertex.

To establish~\eqref{eq:vertex-lower-bound}, let $(u,v)\in T$ be arbitrary. By definition of
convex hull, there exist coefficients $\lambda_1,\lambda_2,\lambda_3\ge 0$ with
$\lambda_1+\lambda_2+\lambda_3=1$ such that
\[
(u,v)\ =\ \lambda_1 w^{(1)} + \lambda_2 w^{(2)} + \lambda_3 w^{(3)}.
\]
Concavity of $f_r$ means precisely that for any such convex combination,
\begin{equation}
\label{eq:concavity}
f_r\!\left(\lambda_1 w^{(1)} + \lambda_2 w^{(2)} + \lambda_3 w^{(3)}\right)
\ \ge\
\lambda_1 f_r(w^{(1)}) + \lambda_2 f_r(w^{(2)}) + \lambda_3 f_r(w^{(3)}).
\end{equation}
Using the representation of $(u,v)$ and applying~\eqref{eq:concavity} gives
\[
f_r(u,v)
\ \ge\
\lambda_1 f_r(w^{(1)}) + \lambda_2 f_r(w^{(2)}) + \lambda_3 f_r(w^{(3)}).
\]
Finally, since $\lambda_1,\lambda_2,\lambda_3\ge 0$ and sum to $1$, the right-hand side is a
convex combination of the three numbers $\{f_r(w^{(1)}),f_r(w^{(2)}),f_r(w^{(3)})\}$ and
therefore cannot be smaller than their minimum. Formally,
\[
\lambda_1 f_r(w^{(1)}) + \lambda_2 f_r(w^{(2)}) + \lambda_3 f_r(w^{(3)})
\ \ge\
\min_{j\in\{1,2,3\}} f_r(w^{(j)}).
\]
Combining the two inequalities yields~\eqref{eq:vertex-lower-bound}.

\medskip
\noindent We want to now conclude the equality of minima.
Let $j^\star\in\arg\min_{j\in\{1,2,3\}} f_r(w^{(j)})$. Since $w^{(j^\star)}\in T$, we have
\[
\min_{(u,v)\in T} f_r(u,v)\ \le\ f_r\!\left(w^{(j^\star)}\right)\ =\ \min_{j} f_r(w^{(j)}).
\]
On the other hand, inequality~\eqref{eq:vertex-lower-bound} holds for all $(u,v)\in T$, so
taking the minimum over $(u,v)\in T$ on both sides gives
\[
\min_{(u,v)\in T} f_r(u,v)\ \ge\ \min_{j} f_r(w^{(j)}).
\]
Together these two inequalities imply
\[
\min_{(u,v)\in T} f_r(u,v)\ =\ \min_{j\in\{1,2,3\}} f_r(w^{(j)}),
\]
proving the first claim.

\medskip
\noindent Step 4: Translate the result to probabilities.
Exponentiation preserves order, hence from
$f_r(u,v)\ge \min_j f_r(w^{(j)})$ we obtain
\[
\PP_{\bp(u,v)}(r)\ =\ e^{f_r(u,v)}
\ \ge\
\exp\!\Big(\min_j f_r(w^{(j)})\Big)
\ =\
\min_j \exp(f_r(w^{(j)}))
\ =\
\min_{j\in\{1,2,3\}} \PP_{\bp(w^{(j)})}(r).
\]
This is exactly $\PP_{\bp(u,v)}(r)\ge \underline{P}_T(r)$ for all $(u,v)\in T$, completing the proof.
\end{proof}
\noindent Generalization to $k>3$: The same argument holds verbatim in dimension $d=k-1$ when $T$ is replaced by a simplex
$S\subset\mathbb{R}^{k-1}$ with vertices $\{w^{(j)}\}_{j=1}^{k}$, since every point in a
simplex is a convex combination of its vertices and $\log\PP_{\bp(u)}(r)$ remains concave in
$u$.
\subsection{Detailed Proof of Proposition~\ref{prop:pval-interval}}
\label{app:pval-interval-proof}

\noindent \textbf{Restatement of Proposition~\ref{prop:pval-interval}}.
\label{prop:pval-interval-detailed}
Fix an observed outcome $\hat r\in\N^3$ with $\hat r_1+\hat r_2+\hat r_3=n$ and let $\widehat{\bp}=\hat r/n\in\Delta_{3,n}$. Let $T\subset\mathbb{R}^2$ be a triangle with vertices $w^{(1)},w^{(2)},w^{(3)}$. For each outcome $r\in\N^3$ with sum $n$, define
\[
g_{r;\hat r}(u,v)
:=
(r_1-\hat r_1)u + (r_2-\hat r_2)v - \big(\log\kappa(\hat r)-\log\kappa(r)\big),
\]
and the index sets
\[
\mathcal{I}_T(\hat r)
:=
\{r:\ g_{r;\hat r}(w^{(j)})\le 0\ \text{for all }j=1,2,3\},
\qquad
\mathcal{O}_T(\hat r)
:=
\{r:\ g_{r;\hat r}(w^{(j)})> 0\ \text{for all }j=1,2,3\}.
\]
Let
\[
\underline{P}_T(r) := \min_{j\in\{1,2,3\}} \PP_{\bp(w^{(j)})}(r),
\]
and define
\[
\underline{\rho}_T(\hat r)
:=
\sum_{r\in\mathcal{I}_T(\hat r)} \underline{P}_T(r),
\qquad
\overline{\rho}_T(\hat r)
:=
1-\sum_{r\in\mathcal{O}_T(\hat r)} \underline{P}_T(r).
\]
Then for every $(u,v)\in T$,
\[
\underline{\rho}_T(\hat r)\ \le\ \rho_{\widehat{\bp}}(\bp(u,v))\ \le\ \overline{\rho}_T(\hat r),
\]
where
\[
\rho_{\widehat{\bp}}(\bp(u,v))
=
\sum_{r:\ \PP_{\bp(u,v)}(r)\le \PP_{\bp(u,v)}(\hat r)} \PP_{\bp(u,v)}(r)
\]
is the exact p-value defined in~\eqref{eqn:partial}.

\begin{proof}
Fix a triangle $T$ with vertices $w^{(1)},w^{(2)},w^{(3)}$. Throughout, let $(u,v)\in T$
be arbitrary.

\medskip
\noindent We start by characterizing the halfspace tail membership.
By Lemma~\ref{lem:halfspace-3}, for any outcome $r$ we have the equivalence
\begin{equation}
\label{eq:tail-equiv}
\PP_{\bp(u,v)}(r)\le \PP_{\bp(u,v)}(\hat r)
\quad\Longleftrightarrow\quad
g_{r;\hat r}(u,v)\le 0.
\end{equation}
Moreover, $g_{r;\hat r}(u,v)$ is an affine (in fact linear plus constant) function of
$(u,v)$.

\medskip
\noindent Then we can clarify the vertices for ``definitely-in-tail'' outcomes.
Consider any $r\in\mathcal{I}_T(\hat r)$. By definition, $g_{r;\hat r}(w^{(j)})\le 0$ for
all vertices $j=1,2,3$. Since $g_{r;\hat r}$ is affine and $T$ is the convex hull of its
vertices, every point $(u,v)\in T$ can be written as a convex combination
\[
(u,v)=\lambda_1 w^{(1)}+\lambda_2 w^{(2)}+\lambda_3 w^{(3)},
\qquad
\lambda_j\ge 0,\ \ \sum_{j=1}^3 \lambda_j=1.
\]
Affine functions preserve convex combinations, hence
\[
g_{r;\hat r}(u,v)
=
\lambda_1 g_{r;\hat r}(w^{(1)})+\lambda_2 g_{r;\hat r}(w^{(2)})+\lambda_3 g_{r;\hat r}(w^{(3)})
\ \le\ 0.
\]
Using~\eqref{eq:tail-equiv}, we conclude that for every $(u,v)\in T$,
\begin{equation}
\label{eq:def-in-tail}
r\in\mathcal{I}_T(\hat r)
\quad\Longrightarrow\quad
\PP_{\bp(u,v)}(r)\le \PP_{\bp(u,v)}(\hat r),
\end{equation}
i.e., $r$ belongs to the p-value tail for \emph{all} parameters in the triangle.

\medskip
\noindent Now consider the vertices for ``definitely-out-of-tail'' outcomes.
Now consider any $r\in\mathcal{O}_T(\hat r)$. By definition, $g_{r;\hat r}(w^{(j)})>0$
for all vertices. The same convex-combination argument implies that for every $(u,v)\in T$,
\[
g_{r;\hat r}(u,v)
=
\lambda_1 g_{r;\hat r}(w^{(1)})+\lambda_2 g_{r;\hat r}(w^{(2)})+\lambda_3 g_{r;\hat r}(w^{(3)})
\ >\ 0,
\]
and hence, by~\eqref{eq:tail-equiv},
\begin{equation}
\label{eq:def-out-tail}
r\in\mathcal{O}_T(\hat r)
\quad\Longrightarrow\quad
\PP_{\bp(u,v)}(r)>\PP_{\bp(u,v)}(\hat r)
\end{equation}
for all $(u,v)\in T$. In particular, such $r$ is excluded from the p-value tail
\emph{everywhere} on $T$.

\medskip
    \noindent We then derive a uniform lower bound on probabilities over $T$.
By Lemma~\ref{lem:vertex-min}, for each outcome $r$ and for all $(u,v)\in T$,
\begin{equation}
\label{eq:prob-lower}
\PP_{\bp(u,v)}(r)\ \ge\ \underline{P}_T(r)
\ :=\ \min_{j\in\{1,2,3\}}\PP_{\bp(w^{(j)})}(r).
\end{equation}

\medskip
\noindent Lower bound on the p-value as follow:
Let $\mathcal{T}(u,v)$ denote the p-value tail set at $(u,v)$:
\[
\mathcal{T}(u,v)
:=
\{r:\ \PP_{\bp(u,v)}(r)\le \PP_{\bp(u,v)}(\hat r)\}.
\]
By~\eqref{eq:def-in-tail}, every $r\in\mathcal{I}_T(\hat r)$ lies in $\mathcal{T}(u,v)$
for all $(u,v)\in T$, hence
\[
\rho_{\widehat{\bp}}(\bp(u,v))
=
\sum_{r\in\mathcal{T}(u,v)} \PP_{\bp(u,v)}(r)
\ \ge\
\sum_{r\in\mathcal{I}_T(\hat r)} \PP_{\bp(u,v)}(r).
\]
Applying the uniform probability lower bound~\eqref{eq:prob-lower} term-by-term yields
\[
\rho_{\widehat{\bp}}(\bp(u,v))
\ \ge\
\sum_{r\in\mathcal{I}_T(\hat r)} \underline{P}_T(r)
\ =\
\underline{\rho}_T(\hat r).
\]
Since $(u,v)\in T$ was arbitrary, this proves the lower bound uniformly over the triangle:
\[
\underline{\rho}_T(\hat r)\ \le\ \rho_{\widehat{\bp}}(\bp(u,v))
\qquad \forall (u,v)\in T.
\]

\medskip
\noindent Similarly, we can upper bound the p-value.
By~\eqref{eq:def-out-tail}, every $r\in\mathcal{O}_T(\hat r)$ is \emph{not} in the tail
$\mathcal{T}(u,v)$ for any $(u,v)\in T$. Therefore the p-value tail is contained in the
complement of $\mathcal{O}_T(\hat r)$:
\[
\mathcal{T}(u,v)\ \subseteq\ \Delta_{3,n}\setminus \mathcal{O}_T(\hat r).
\]
Hence
\begin{align}
\rho_{\widehat{\bp}}(\bp(u,v))
=
\sum_{r\in\mathcal{T}(u,v)} \PP_{\bp(u,v)}(r)
\ \le\
\sum_{r\in \Delta_{3,n}\setminus \mathcal{O}_T(\hat r)} \PP_{\bp(u,v)}(r)
=
1-\sum_{r\in\mathcal{O}_T(\hat r)} \PP_{\bp(u,v)}(r).
\label{eq:upper-start}
\end{align}
Applying the uniform lower bound~\eqref{eq:prob-lower} to each $r\in\mathcal{O}_T(\hat r)$
gives
\[
\sum_{r\in\mathcal{O}_T(\hat r)} \PP_{\bp(u,v)}(r)
\ \ge\
\sum_{r\in\mathcal{O}_T(\hat r)} \underline{P}_T(r),
\]
and substituting into~\eqref{eq:upper-start} yields
\[
\rho_{\widehat{\bp}}(\bp(u,v))
\ \le\
1-\sum_{r\in\mathcal{O}_T(\hat r)} \underline{P}_T(r)
\ =\
\overline{\rho}_T(\hat r).
\]
Again, since $(u,v)\in T$ was arbitrary, this holds uniformly for all points in the triangle.

\medskip
\noindent Combining the two bonds establishs that for all $(u,v)\in T$,
\[
\underline{\rho}_T(\hat r)\ \le\ \rho_{\widehat{\bp}}(\bp(u,v))\ \le\ \overline{\rho}_T(\hat r),
\]
completing the proof.
\end{proof}

Outcomes in $\mathcal{I}_T(\hat r)$ are guaranteed to contribute to the p-value tail everywhere on $T$, and $\underline{P}_T(r)$ lower-bounds their contribution uniformly, yielding $\underline{\rho}_T(\hat r)$. Outcomes in $\mathcal{O}_T(\hat r)$ are guaranteed to never contribute on $T$; by lower-bounding their probability mass we upper-bound the remaining tail probability by subtraction from $1$, yielding $\overline{\rho}_T(\hat r)$.
\subsection{Example of Zero-count exclusion.}
\label{app:example-zero-count}
Consider $k=3$ categories and $n=5$ samples. Suppose the observed outcome is $\hat r=(2,2,1)$, so that category~3 appears at least once. Now consider a parameter $\bp=(1/2,\,1/2,\,0)$ assigning zero probability to category~3. Then
\[
\PP_{\bp}(\hat r)
=\frac{5!}{2!\,2!\,1!}(1/2)^2(1/2)^2(0)^1
=0.
\]
Since $\PP_{\bp}(\hat r)=0$, the p-value tail consists only of outcomes of probability
zero, and hence $\rho_{\hat r}(\bp)=0$. In particular, $\bp$ cannot belong to the MVC at
any confidence level $\alpha>0$.

By contrast, if $\hat r=(2,3,0)$ and $\bp=(1/2,\,1/2,\,0)$, then $\PP_{\bp}(\hat r)>0$ and the p-value is not forced to vanish. This illustrates why only categories with zero count in \emph{both} outcomes may be set to zero when searching for intersection witnesses. When a category $i$ has zero count in both observed outcomes, setting $p_i=0$ does not invalidate either likelihood. In this case the multinomial model \emph{collapses} onto a lower-dimensional face of the simplex, corresponding to a reduced multinomial over the remaining categories. Importantly, this collapse is \textbf{EXACT} rather than approximate: the distribution of outcomes conditional on $p_i=0$ is precisely a multinomial with one fewer
category and the same sample size $n$.

From the perspective of the p-value, this means that restricting $\bp$ to a face $\Delta_k^{(T)}$ (for $T\subseteq Z$) induces the same p-value ordering as that obtained by deleting the zeroed coordinates from both the parameter and the observed outcome. As a
result, testing MVC intersection on a face $\Delta_k^{(T)}$ is equivalent to running the same certified intersection test on the reduced outcomes in dimension $k-|T|$.

Algorithm~\ref{alg:faces} exploits this structure by enumerating all faces corresponding to jointly-zero categories and applying the interior algorithm on each reduced problem. In this way, boundary behavior is handled exactly, without approximation or special-case logic, while preserving the same soundness guarantees as in the interior.

\subsection{Example of collapsing onto a lower-dimensional face of the simplex}

\label{app:example-collapsing}

Consider an AB testing problem with $k=4$ categories and sample size $n=5$. Suppose the
observed outcomes are
\[
\hat r^{(A)} = (3,2,0,0),
\qquad
\hat r^{(B)} = (4,1,0,0).
\]
Categories $1$ and $2$ appear at least once in both outcomes, while categories $3$ and $4$
are never observed.

Accordingly, we define
\[
S_0 = \{1,2\},
\qquad
Z = \{3,4\}.
\]

Suppose we consider a parameter $\bp$ with $p_1=0$. Since $\hat r^{(A)}_1>0$ and
$\hat r^{(B)}_1>0$, the multinomial likelihood of both observed outcomes under $\bp$ is zero.
By Lemma~\ref{lem:zero-exclusion}, this implies
\[
\rho_A(\bp)=\rho_B(\bp)=0.
\]
Hence $\bp$ cannot belong to either minimum-volume confidence set, and no intersection
witness can lie on any face with $p_1=0$. The same reasoning applies to $p_2=0$.

Thus, categories in $S_0$ must remain strictly positive in any candidate intersection point.

In contrast, categories $3$ and $4$ have zero count in \emph{both} outcomes. Setting
$p_3=0$ or $p_4=0$ does not force the likelihood of either observed outcome to vanish.
Consequently, the p-values $\rho_A(\bp)$ and $\rho_B(\bp)$ may remain strictly positive,
and such parameters cannot be ruled out a priori.

Each subset $T\subseteq Z$ corresponds to a face of the simplex on which
\[
p_i=0 \quad \forall i\in T.
\]
In this example, the relevant faces are:
\begin{itemize}
\item $T=\emptyset$: the full interior with $p_1,p_2,p_3,p_4>0$;
\item $T=\{3\}$: the face $p_3=0$;
\item $T=\{4\}$: the face $p_4=0$;
\item $T=\{3,4\}$: the face $p_3=p_4=0$, reducing the problem to a two-category (binomial)
model over categories $1$ and $2$.
\end{itemize}

Every parameter $\bp\in\Delta_4$ lies in the relative interior of exactly one of these faces. Algorithm~\ref{alg:faces} checks each face exhaustively using the certified intersection test in the corresponding reduced dimension.

\begin{algorithm}[t]
\caption{Certified MVC intersection test with zero-count handling}
\label{alg:faces}
\begin{algorithmic}[1]
\Require Observed outcomes $\hat r^{(A)},\hat r^{(B)}\in\N^k$, level $\alpha\in(0,1)$,
decision margin $\tau>0$, refinement tolerance $\varepsilon>0$.
\Ensure \textsf{INTERSECT}, \textsf{DISJOINT}, or \textsf{UNCERTAIN}.

\State Compute $S_0=\{i:\hat r^{(A)}_i>0\}\cup\{i:\hat r^{(B)}_i>0\}$ and
$Z=\{1,\ldots,k\}\setminus S_0$.

\For{each subset $T\subseteq Z$}
    \State Form reduced outcomes $\hat r^{(A)}_{-T},\hat r^{(B)}_{-T}$ by deleting
    coordinates in $T$.
    \State Set $k' \gets k-|T|$.
    \State Construct a compact search domain $\cW_T\subset\R^{k'-1}$.
    \State Run Algorithm~\ref{alg:k3} (or its $(k'-1)$-dimensional generalization)
    on $(\hat r^{(A)}_{-T},\hat r^{(B)}_{-T})$.
    \If{output is \textsf{INTERSECT}}
        \State \Return \textsf{INTERSECT}.
    \EndIf
\EndFor

\If{all face subproblems returned \textsf{DISJOINT}}
    \State \Return \textsf{DISJOINT}.
\Else
    \State \Return \textsf{UNCERTAIN}.
\EndIf
\end{algorithmic}
\end{algorithm}

\subsection{Bounding box construction for $W$ in $(u,v)$-coordinates}\label{sec:W-bounding-box}

We provide two ways to derive a bound for  $W$, one is closed-form and the other is analytical.
Throughout we start with $k=3$, with log-odds coordinates
\[
u=\log\frac{p_1}{p_3},\qquad v=\log\frac{p_2}{p_3},
\]
so that
\[
p_1(u,v)=\frac{e^u}{1+e^u+e^v},\quad
p_2(u,v)=\frac{e^v}{1+e^u+e^v},\quad
p_3(u,v)=\frac{1}{1+e^u+e^v}.
\]
For an observed count vector $\hat r=(\hat r_1,\hat r_2,\hat r_3)$ with $\sum_i \hat r_i=n$, define
\[
f_{\hat r}(u,v)\triangleq\log P_{p(u,v)}(\hat r)
= \log\kappa(\hat r) + \hat r_1 u + \hat r_2 v
- n\log\!\bigl(1+e^u+e^v\bigr),
\]
where $\kappa(\hat r)=\frac{n!}{\hat r_1!\hat r_2!\hat r_3!}$. For a fixed threshold $t>0$, the superlevel set
\[
S_{\hat r}(t)\triangleq
\{(u,v)\in\mathbb{R}^2:\; P_{p(u,v)}(\hat r)\ge t\}
=
\{(u,v):\; f_{\hat r}(u,v)\ge \log t\}
\]
is convex (indeed, compact), since $f_{\hat r}$ is concave in $(u,v)$.

In the certified intersection procedure we require a compact search domain containing all possible witnesses. A convenient sufficient condition is
\[
W\subseteq S_{\hat r^{(A)}}(t)\cap S_{\hat r^{(B)}}(t),
\qquad
t=\frac{\alpha-\tau}{m},\quad m=\lvert\Delta_{3,n}\rvert.
\]

An axis-aligned bounding box for $S_{\hat r}(t)$ (or for the intersection of two such sets) requires computing coordinate extrema such as
\[
u_{\max}=\sup\{u:\exists v\text{ s.t. }f_{\hat r}(u,v)\ge\log t\},
\qquad
v_{\max}=\sup\{v:\exists u\text{ s.t. }f_{\hat r}(u,v)\ge\log t\},
\]
and analogously $u_{\min},v_{\min}$. Boundary points satisfy the level-set equation
\[
\log\kappa(\hat r)+\hat r_1 u+\hat r_2 v
- n\log(1+e^u+e^v)
= \log t.
\]
Rearranging gives
\[
1+e^u+e^v
=
\exp\!\left(
\frac{\log\kappa(\hat r)+\hat r_1 u+\hat r_2 v-\log t}{n}
\right),
\]
which mixes $e^u$, $e^v$, and an exponential of an affine form in $(u,v)$. Solving for one variable given the other leads to a transcendental equation; no algebraic closed form exists in general. The same obstruction persists when bounding the intersection
\[
S_{\hat r^{(A)}}(t)\cap S_{\hat r^{(B)}}(t),
\]
whose extrema arise from convex programs with boundary conditions of the same form.

Although the boundary itself is not closed-form, each univariate slice admits a closed-form maximizer. Fix $u$ and consider $v\mapsto f_{\hat r}(u,v)$. Differentiation yields
\[
\frac{\partial f_{\hat r}}{\partial v}(u,v)
=
\hat r_2 - n\frac{e^v}{1+e^u+e^v}.
\]
Setting this derivative to zero gives
\[
\frac{e^v}{1+e^u+e^v}=\frac{\hat r_2}{n}
\quad\Longrightarrow\quad
e^v=\frac{\hat r_2}{n-\hat r_2}(1+e^u),
\]
and hence the unique slice maximizer
\[
v^*(u)=
\log\!\left(\frac{\hat r_2(1+e^u)}{n-\hat r_2}\right).
\]
By concavity, for fixed $u$ the superlevel set
\[
\{v:\;f_{\hat r}(u,v)\ge\log t\}
\]
is either empty or an interval
\[
[v_-(u),v_+(u)],
\]
with endpoints satisfying $f_{\hat r}(u,v)=\log t$. These endpoints have no closed form but can be computed robustly by one-dimensional bracketing and bisection. The analogous statements hold for $u^*(v)$ obtained by symmetry.

If a purely closed-form (but very loose) bound is desired, one may use the inequality
\[
1+e^u+e^v \le 3e^{\max\{0,u,v\}},
\]
which implies
\[
\log(1+e^u+e^v)\le \log 3+\max\{0,u,v\}.
\]
Substituting into $f_{\hat r}$ yields
\[
f_{\hat r}(u,v)
\ge
\log\kappa(\hat r)
+\hat r_1 u+\hat r_2 v
- n\log 3
- n\max\{0,u,v\}.
\]
Thus any $(u,v)$ satisfying
\[
\log\kappa(\hat r)
+\hat r_1 u+\hat r_2 v
- n\log 3
- n\max\{0,u,v\}
\ge \log t
\]
is guaranteed to lie in $S_{\hat r}(t)$. This condition is piecewise linear and can be unfolded into explicit linear inequalities by considering the cases $\max\{0,u,v\}=0$, $u$, or $v$.

Conversely, using the lower bound
\[
\log(1+e^u+e^v)\ge \max\{0,u,v\},
\]
we obtain the implication
\[
f_{\hat r}(u,v)
\le
\log\kappa(\hat r)
+\hat r_1 u+\hat r_2 v
- n\max\{0,u,v\}.
\]
Whenever the right-hand side is strictly less than $\log t$, the point $(u,v)$ is certainly outside $S_{\hat r}(t)$. This yields a closed-form (but conservative) exclusion region, from which an explicit axis-aligned outer box can be derived.

\end{document}